\documentclass[final,12pt]{alt2021} 

\usepackage{microtype}
\usepackage{graphicx}
\usepackage{booktabs} 

\usepackage{hyperref}

\usepackage{footmisc}

\usepackage[]{natbib}

\usepackage{url}            
\usepackage{booktabs}       
\usepackage{amsfonts}       
\usepackage{nicefrac}       
\usepackage{microtype}      
\usepackage{amsmath}
\usepackage{mathtools}

\usepackage{enumitem}
\usepackage{dirtytalk}
\usepackage{tikz}
\usetikzlibrary{arrows.meta,calc}
\usepackage{dsfont}
\usepackage{amssymb}
\usepackage{upgreek}
\usepackage{multirow}
\usepackage{soul}
\usepackage{framed}
\usepackage{pifont}
\usepackage{xspace}
\usepackage{float}
\usepackage{empheq}

\usepackage{multicol}
\usepackage{textgreek}
\makeatletter
\newcommand{\greekalpha}[1]{\c@greekalpha{#1}}
\newcommand{\c@greekalpha}[1]{%
  {%
    \ifcase\number\value{#1} %
    \or
    \textalpha
    \or
    \textbeta
    \or
    \textgamma
    \fi
  }%
}

\AddEnumerateCounter*{\greekalpha}{\c@greekalpha}{5}

\newcommand{\CommaBin}{\mathbin{\raisebox{0.5ex}{,}}}
\let\originalleft\left
\let\originalright\right
\renewcommand{\left}{\mathopen{}\mathclose\bgroup\originalleft}
\renewcommand{\right}{\aftergroup\egroup\originalright}
\newcommand{\wt}[1]{\widetilde{#1}}

\newcommand{\wh}[1]{\widehat{#1}}

\newcommand{\cS}{\mathcal{S}}
\newcommand{\cA}{\mathcal{A}}

\newcommand{\SampCompOne}{\textsc{SampComp}(\epsilon)}
\newcommand{\SampCompTwo}{\textsc{SampComp}(\epsilon, \theta)}

\newcommand{\EVI}{{\small\textsc{EVI}}\xspace}
\newcommand{\VI}{{\small\textsc{VI}}\xspace}

\newcommand{\UCBVI}{{\small\textsc{UCBVI}}\xspace}

\newcommand\myineeqa{\mathrel{\stackrel{\makebox[0pt]{\mbox{\normalfont\tiny (a)}}}{\leq}}}
\newcommand\myineeqb{\mathrel{\stackrel{\makebox[0pt]{\mbox{\normalfont\tiny (b)}}}{\leq}}}
\newcommand\myineeqc{\mathrel{\stackrel{\makebox[0pt]{\mbox{\normalfont\tiny (c)}}}{\leq}}}
\newcommand\myineeqd{\mathrel{\stackrel{\makebox[0pt]{\mbox{\normalfont\tiny (d)}}}{\leq}}}

\newcommand\ineqcs{\mathrel{\stackrel{\makebox[0pt]{\mbox{\normalfont\tiny (C-S)}}}{\leq}}}

\makeatletter
\newcommand\footnoteref[1]{\protected@xdef\@thefnmark{\ref{#1}}\@footnotemark}
\makeatother

\DeclarePairedDelimiter\abs{\lvert}{\rvert}%
\DeclarePairedDelimiter\norm{\lVert}{\rVert}%

\newcommand{\gammaVI}{\mu_{\textsc{{\tiny VI}}}\xspace}

\renewcommand{\epsilon}{\varepsilon}

\renewcommand{\tilde}{\widetilde}

\makeatletter
\renewcommand*{\@fnsymbol}[1]{\ensuremath{\ifcase#1\or \dagger \or \mathsection \or *\or \dagger\or \ddagger\or
    \mathsection\or \mathparagraph\or \|\or **\or \dagger\dagger
    \or \ddagger\ddagger \else\@ctrerr\fi}}
\makeatother

\title[Sample Complexity Bounds for Stochastic Shortest Path with a Generative Model]{Sample Complexity Bounds for Stochastic Shortest Path \\ with a Generative Model}
\usepackage{times}

\altauthor{%
 \Name{Jean Tarbouriech} \Email{jean.tarbouriech@gmail.com}\\
 \addr Facebook AI Research Paris \& Inria Lille
 \AND
 \Name{Matteo Pirotta} \Email{pirotta@fb.com}\\
 \addr Facebook AI Research Paris%
 \AND
 \Name{Michal Valko} \Email{valkom@deepmind.com}\\
 \addr DeepMind Paris%
 \AND
 \Name{Alessandro Lazaric} \Email{lazaric@fb.com}\\
 \addr Facebook AI Research Paris%
}

\begin{document}

\maketitle

\begin{abstract}%
We consider the objective of computing an $\epsilon$-optimal policy in a stochastic shortest path (SSP) setting, provided that we can access a generative sampling oracle. We  propose two algorithms for this setting and derive PAC bounds on their sample complexity: one for the case of positive costs and the other for the case of non-negative costs under a restricted optimality criterion. While tight sample complexity bounds have been derived for the finite-horizon and discounted MDPs, the SSP problem is a strict generalization of these settings and it poses additional technical challenges due to the fact that no specific time horizon is prescribed and policies may never terminate, i.e., we are possibly facing non-proper policies. As a consequence, we can neither directly apply existing techniques minimizing sample complexity nor rely on a regret-to-PAC conversion leveraging recent regret bounds for SSP. Our analysis instead combines SSP-specific tools and variance reduction techniques to obtain the first sample complexity bounds for this setting.
\end{abstract}


\section{Introduction}

A common assumption in approximate dynamic programming and reinforcement learning (RL) is to have access to a generative model of the Markov decision process (MDP), that is, a sampling device which can generate samples of the transition and reward functions at any state-action pair. A large body of prior work~\citep{azar2013minimax, wang2017randomized, sidford2018near, sidford2018variance, zanette2019almost, agarwal2019model, li2020breaking} studied how to compute an $\epsilon$-optimal policy in the infinite-horizon discounted MDP (DMDP) setting with as few calls to the generative model as possible. While the infinite-horizon discounted setting is common in practical RL, many problems are better formalized within the strictly more general\footnote{\label{footnote1}Any DMDP with discount factor $\gamma$ can indeed be converted into an SSP problem with the same state space augmented by an artificial termination state that is reached with probability $1-\gamma$ at any time step and state-action pair.} stochastic shortest-path (SSP) setting~\citep{bertsekas1995dynamic}, where the objective is to compute a policy that minimizes the cost accumulated before reaching a specific goal state. Recently, \citet{tarbouriech2019no} and \citet{cohen2020near} studied the SSP problem in the online case and derived the first regret bounds for this setting. 

In this paper we focus on the generative model setting and study the problem of computing a near-optimal policy in an SSP problem with a given goal state and cost function. We derive two closely related algorithms for this setting and we prove PAC bounds for their sample complexity. The first algorithm is designed to return an $\epsilon$-optimal policy for any SSP problem with strictly positive cost function and it has a sample complexity that is adaptive to the (unknown) range of the optimal value function. The second can be used for any cost function, including the case when the cost is zero in some states, for which the optimal policy may not even be proper (i.e., it may never reach the goal yet still minimize the cumulative cost). In this case, we re-frame the objective as computing an $\epsilon$-optimal solution in the set of (proper) policies that in expectation reach the goal in at most a number of steps that is proportional to the minimum number of expected steps to the goal (i.e., the SSP-diameter of the problem). The main technical challenge in deriving these results is due to the fact that in SSP we have no knowledge of the \emph{effective} horizon of the problem, unlike in DMDP (with $1/(1-\gamma)$). As a result, model estimation errors may accumulate indefinitely, thus preventing from achieving any desired level of accuracy. In order to deal with this problem, a first approach is to build on an SSP-specific simulation lemma, which reveals the level of accuracy needed in estimating the MDP to be able to recover proper policies and the role played by the minimum cost. Although this approach yields a bound on the sample complexity, we show that it can in fact be tightened by leveraging and combining SSP-specific tools for regret minimization~\citep{cohen2020near} and variance-aware techniques for DMDP sample complexity~\citep{azar2013minimax}.

\section{Preliminaries}
\label{sec_preliminaries}

\paragraph{Stochastic shortest path (SSP)}
We start by introducing the notion of MDP with an SSP objective
as done by \citet[][Sect.\,3]{bertsekas1995dynamic}.
\begin{definition}[SSP-MDP]\label{def:ssp.mdp.app}
    An SSP-MDP is an MDP $$M := \langle \mathcal{S}, \mathcal{A}, g, p, c \rangle,$$ where $\mathcal{S}$ is the state space with $ S := |\mathcal{S}|$ states and $\mathcal{A}$ is the action space with $A := |\mathcal{A}|$ actions. We denote by $g \notin \mathcal{S}$ the goal state, and we set $\cS' := \cS \cup \{g\}$. Taking action $a$ in state $s$ incurs a cost of $c(s,a) \in [0,1]$ and the next state $s' \in \cS'$ is selected with probability $p(s' \vert s, a)$. 
    The goal state $g$ is absorbing and zero-cost, i.e., $p(g \vert g,a) = 1$ and $c(g,a)=0$ for any action $a \in \cA$, which effectively implies that the agent ends its interaction with $M$ when reaching the goal $g$.
\end{definition}
We denote by $\Pi := \{\pi : \mathcal{S} \to \mathcal{A}\}$ the set of stationary deterministic policies. For any $\pi \in \Pi$ and $s \in \mathcal{S}$, the random (possibly unbounded) goal-reaching time starting from $s$ is denoted by $\tau_{\pi}(s) := \inf \{ t \geq 0: s_{t+1} = g \,\vert\, s_1 = s, \pi \}$. 
\begin{definition}[Proper policy]\label{def:proper}
    A policy $\pi$ is \textit{proper} if its execution reaches the goal with probability~$1$ when starting from any state in $\cS$. A policy is \textit{improper} if it is not proper. The set of proper policies is denoted by $\Pi_{\textrm{proper}}$.
\end{definition}
\begin{assumption}
There exists at least one proper policy, i.e., $\Pi_{\textrm{proper}} \neq \emptyset$.
\label{asm_proper}
\end{assumption}
The value function (also called expected cost-to-go) of a policy $\pi \in \Pi$ is defined as
\begin{align*}
    V^{\pi}(s) &:= \mathbb{E}\bigg[ \sum_{t = 1}^{+\infty} c(s_{t}, \pi(s_t)) \,\Big\vert\,s_1 = s\bigg] = \mathbb{E}\bigg[ \sum_{t = 1}^{\tau_{\pi}(s)} c(s_{t}, \pi(s_t)) \,\Big\vert\,s_1 = s\bigg],
\end{align*}
where the expectation is w.r.t.\,the random sequence of states generated by executing $\pi$ starting from state $s \in \mathcal{S}$. Note that for improper policies $\pi \not\in \Pi_{\textrm{proper}},$ $V^{\pi}$ has at least one unbounded component. Our objective is to find an optimal policy $\pi^{\star}$ that minimizes the value function. For any vector $V \in \mathbb{R}^S$, the optimal Bellman operator is defined as
\begin{align}\label{bellman_eq}
\mathcal{L}V(s) := \min_{a \in \mathcal{A}} \Big\{ c(s, a) + \sum_{y \in \mathcal{S}} p(y  \, \vert \,  s,a) V(y) \Big\}. 
\end{align}
\begin{lemma}[\citealp{bertsekas1991analysis}, Prop.\,2]
Suppose that Asm.\,\ref{asm_proper} holds and that for every improper policy $\pi'$ there exists at least one state $s \in \cS$ such that $V^{\pi'}(s) = + \infty$. Then the optimal policy $\pi^{\star}$ is stationary, deterministic, and proper. Moreover, $V^\star = V^{\pi^\star}$ is the unique solution of the optimality equations $V^\star = \mathcal{L} V^\star$ and $V^\star(s) < +\infty$ for any $s \in \mathcal{S}$.
\label{lemma_wellposedproblem}
\end{lemma}
We define the following quantities:
\begin{itemize}[leftmargin=0.2in]
    \item The \textit{SSP-diameter} $D$ \citep{tarbouriech2019no} is defined as
\begin{align}\label{definition_SSP_diameter}
    D := \max_{s \in \mathcal{S}} D_s, \quad~ \textrm{with}~~ D_s := \min_{\pi \in \Pi} \mathbb{E}\left[\tau_{\pi}(s)\right].
\end{align}
    \item $B_{\star} := \max_{s \in \cS} \min_{\pi \in \Pi_{\textrm{proper}}} V^{\pi}(s)$ is the maximal optimal value function over states.
    \item $\Gamma := \max_{s,a} \norm{p(\cdot \vert s,a)}_0 \leq S+1$ is the maximal support of $p(\cdot \vert s,a)$. 
    \item Finally, $c_{\min} := \min_{(s,a) \in \cS \times \cA} c(s,a) \in [0,1]$ is the minimum non-goal cost.
\end{itemize}
Note that since the number of states $S$ is finite, Asm.\,\ref{asm_proper} implies that $B_{\star} \leq D<+\infty$.

\paragraph{Problem formulation}
We consider that the costs $c$ are deterministic and known to the learner, while the transition dynamics $p$ is unknown. We assume access to a generative model, which for any state-action pair $(s,a)$ returns a sample drawn from $p(\cdot|s,a)$. We ask the following: \textit{How many calls to the generative model are sufficient to compute a near-optimal policy with high probability?}
\paragraph{On the online-to-batch conversion in the SSP setting} Since the problem of learning in SSP has already been studied in the regret-minimization setting~\citep{tarbouriech2019no, cohen2020near}, it may be tempting to leverage a regret-to-PAC conversion to obtain a sample complexity bound and provide a first answer to the question above. For instance, in finite-horizon MDPs, a regret bound can be converted to a PAC guarantee by selecting as a candidate optimal solution any policy chosen at random out of all episodes~\citep{Jin2018qlearning}. Unfortunately this procedure cannot be applied here. In fact, the SSP-regret differs from the finite-horizon regret, since at each episode it compares the \textit{empirical} costs accumulated along one trajectory with the optimal value function. Indeed, following $K$ episodes with initial state $s_0$ where for each $k\in [K]$ the (possibly non-stationary) policy executed is denoted by $\pi_k$, we recall that the SSP-regret is defined as
\begin{align*}
    \Big[ \sum_{k=1}^K \sum_{h=1}^{\tau_{\pi_k}(s_0)} c(s_{k,h}, \pi_k(s_{k,h})) \Big] - K \cdot \min_{\pi \in \Pi_{\textrm{proper}}} V^{\pi}(s_0),
\end{align*}
where $s_{k,h}$ denotes the $h$-th state visited during episode $k$ and $\tau_{\pi_k}(s_0)$ is the (possibly infinite) time it takes the agent to complete episode $k$. As a result, no guarantee is provided for the value function (i.e., the expected cumulative costs) of one episode. Indeed, it has been shown in the existing regret analyses for SSP that explicitly guaranteeing the properness of the deployed policies is \textit{not} an intermediate step that is required to derive the regret bounds.  SSP-regret algorithms may change multiple policies within each episode and none of them may actually be proper (i.e., they may have an unbounded value function, so that there exists a state $s$ such that $V^\pi(s) = +\infty$). As such, it is unclear which policy should be retained as a solution candidate. Finally, the near-optimal guarantees we intend to achieve are for any arbitrary initial state in $\mathcal{S}$, while regret-style guarantees are only in expectation with respect to a starting state distribution. 

\paragraph{Distinction of cases depending on $c_{\min}$} Our analysis considers two distinct cases, $c_{\min} > 0$ and $c_{\min} = 0$, for which it targets different notions of sample complexity.

\paragraph{First case: sample complexity objective for $c_{\min} > 0$}
When the cost function is strictly positive, the conditions of Lem.\,\ref{lemma_wellposedproblem} hold, so the optimal policy is guaranteed to be proper. We seek to achieve the following standard PAC guarantees.
\begin{definition}\label{def_eps_opt}
    We say that an algorithm is $(\epsilon,\delta)$-optimal with sample complexity $n$, if after $n$ calls to the generative model it returns a policy $\pi$ that verifies $\norm{ V^{\pi} - V^{\star}}_{\infty} \leq \epsilon$ with probability at least~$1-\delta$. We denote by $\SampCompOne$ the corresponding sample complexity $n$.
\end{definition}

\paragraph{Second case: sample complexity objective for $c_{\min} = 0$}
The case of zero minimum cost is a complex SSP problem where the optimal policy may not even be guaranteed to be proper \citep{bertsekas1995dynamic}. In this case, it is unclear whether the sample complexity of Def.~\ref{def_eps_opt} may even be bounded, since estimation errors may propagate indefinitely. On the other hand, we seek $\epsilon$-optimality guarantees w.r.t.\,a set of proper policies.  

\begin{definition}[Restricted set $\Pi_{\theta}$]\label{def_restriction}
    For any $\theta \in [1, + \infty]$, we define the set $$\Pi_{\theta} := \{ \pi \in \Pi : \forall s \in \cS, \mathbb{E}\left[ \tau_{\pi}(s) \right] \leq \theta D_{s} \}.$$
\end{definition}
Notice that $\Pi_{+ \infty} = \Pi$. Moreover, for any $\theta \in [1, + \infty)$, $\Pi_{\theta}$ only contains proper policies, i.e., $\Pi_{\theta} \subseteq \Pi_{\textrm{proper}}$. Similar to Def.\,\ref{def_eps_opt}, we then reformulate the desired notion of optimality.

\begin{definition}\label{def_eps_opt.theta}
We say that an algorithm is $(\epsilon,\delta,\theta)$-optimal with sample complexity $n$, if after $n$ calls to the generative model it returns a policy $\pi$ that verifies $\norm{ V^{\pi} - V_\theta^{\star}}_{\infty} \leq \epsilon$ with probability at least~$1-\delta$, where $V_\theta^{\star} = \min_{\pi \in \Pi_{\theta}} V^{\pi}$ is the optimal value function restricted to policies in $\Pi_{\theta}$. We denote by $\SampCompTwo$ the corresponding sample complexity $n$.
\end{definition}
While alternative definitions of restricted set may be introduced, we believe that Def.\,\ref{def_restriction} is well-suited for our problem, as it defines the restriction w.r.t.\,$D_s$, a cost-independent quantity describing the difficulty of navigating in the SSP-MDP (Eq.\,\ref{definition_SSP_diameter}). Nonetheless, this poses an additional layer of complexity, since $D_s$ is unknown to the agent, which only receives $\theta$ as additional parameter.

\section{A first approach: Simulation Lemma for SSP}

We begin by stating a general simulation lemma tailored to SSP which is a useful component to derive sample complexity bounds. For any model $p$ and any $\eta > 0$, we introduce the set of models close to $p$ as follows $$\mathcal{P}_{\eta}^{(p)} := \left\{ p' \in \mathbb{R}^{S' \times A \times S'} : ~ \forall (s,a) \in \cS' \times \cA, ~p'(\cdot \vert s,a) \in \Delta(\cS'), ~ \norm{p(\cdot \vert s,a) - p'(\cdot \vert s,a)}_1 \leq \eta \right\}.$$
Up to a slight difference in the way the set $\mathcal{P}_{\eta}^{(p)}$ is defined, leveraging the result of \citet[][App.\,C]{tarbouriech2020improved} yields the following guarantee (see App.\,\ref{app_useful_results}). 

\begin{lemma}[Simulation Lemma for SSP]\label{lemma_simulation_ssp}
    Consider any $\eta > 0$ and any two models $p$ and $p' \in \mathcal{P}_\eta^{(p)}$ such that, for each model, there exists at least one proper policy w.r.t.\,the goal state $g$. Consider a cost function such that $c_{\min} > 0$. Consider any policy $\pi$ that is proper in $p'$, with value function denoted by $V_{\pi}'$, such that the following condition is verified
    \begin{align}
        \eta \norm{V_{\pi}'}_{\infty} \leq 2 c_{\min}.
    \label{eq_key_sim_lemma_ssp}
    \end{align}
    Then $\pi$ is proper in $p$ (i.e., its value function verifies $V_{\pi} < + \infty$ component-wise), and we have
    \begin{align*}
        \forall s \in \cS, ~ V_{\pi}(s ) \leq \left( 1 +  \frac{2\eta \norm{V'_{\pi}}_{\infty}}{c_{\min}} \right) V'_{\pi}(s ),
    \end{align*}
    and conversely,
    \begin{align*}
    \forall s \in \cS, ~ V'_{\pi}(s ) \leq \left(1 + \frac{\eta \norm{V'_{\pi}}_{\infty}}{c_{\min}} \right) V_{\pi}(s ).
    \end{align*}
    Combining the two inequalities above yields
    \begin{align*}
        \norm{ V_{\pi} - V'_{\pi}}_{\infty} \leq \frac{7 \eta \norm{V'_{\pi}}_{\infty}^2}{c_{\min}}\cdot
    \end{align*}
\end{lemma}
For comparison let us now recall the classical simulation lemma for discounted MDPs.
\begin{lemma}[Simulation Lemma for DMDP, see e.g., \citealp{kearns2002near}]\label{lemma_simulation_finite_horizon}%
Consider any two models $p$ and $p' \in \mathcal{P}_{\eta}^{(p)}$ for any $\eta > 0$. Consider as value function in $p$ the expected discounted cumulative reward, i.e., for any policy $\pi$ and state $s \in \cS$, $V_{\pi}(s) := \mathbb{E}\left[\sum_{t=1}^{+ \infty} \gamma^t r(s_t, \pi(s_t)) ~\vert~s_1 = s   \right]$; and $V'_{\pi}$ is the value function in $p'$. Suppose that the instantaneous rewards are known and bounded in $[0,1]$. Then for any policy $\pi$, we have
    \begin{align*}
        \norm{ V_{\pi} - V'_{\pi}}_{\infty} \leq O\left( \frac{\eta}{(1-\gamma)^2} \right)\cdot
    \end{align*}
\end{lemma}
We spell out the key differences between the simulation lemma in the discounted setting (Lem.\,\ref{lemma_simulation_finite_horizon}) and in SSP (Lem.\,\ref{lemma_simulation_ssp}), bringing to light the criticalities in the latter setting. First, a guarantee can only be obtained if the condition \eqref{eq_key_sim_lemma_ssp} is verified, which involves both the accuracy~$\eta$ and the value function of $\pi$ in $p' \in \mathcal{P}_{\eta}^{(p)}$. We observe that the smaller the minimum cost $c_{\min}$, the smaller the accuracy $\eta$ needs to be. Importantly, $c_{\min}$ must be positive and the error scales inversely with it. Indeed, the \say{trajectory length} is captured not by a known hyperparameter $1/(1-\gamma)$ as in DMDPs, but by the ratio between the (a priori unknown) infinity norm of the value function of the policy and the minimum cost $c_{\min}$ (note that this ratio indeed has a time dimension and it upper bounds the expected goal-reaching time of the policy since $\norm{\mathbb{E}\left[\tau_{\pi}\right]}_{\infty} \leq \norm{V_{\pi}}_{\infty} / c_{\min}$). 

Similar to existing approaches in DMDP, we could leverage the simulation lemma to directly derive sample complexity guarantees. More precisely, building on the result of Lem.\,\ref{lemma_simulation_ssp} and plugging it in the algorithms proposed in Sect.\,\ref{sec_main_result} would eventually lead to sample complexities scaling as 
\begin{align}
    \SampCompOne &= \wt{O} \Bigg( \frac{B_{\star}^4 \Gamma S A}{c_{\min}^2 \epsilon^2} + \frac{B_{\star}^2 S^2 A}{c_{\min} \epsilon} \Bigg), \label{initialbound_1} \\
    \SampCompTwo &= \wt{O} \Bigg( \frac{\theta^2 D^2 B_{\star}^4 \Gamma S A}{\epsilon^4} + \frac{\theta D B_{\star}^2 S^2 A}{\epsilon^2} \Bigg). \label{initialbound_2}
\end{align}
In the following section, we will show that these two bounds can in fact be improved by refining the guarantee of Lem.\,\ref{lemma_simulation_ssp} thanks to variance-aware arguments (as also leveraged by e.g., \citealp{azar2013minimax, cohen2020near}). Notably, it will enable to shave off a $B^{\star} / c_{\min}$ dependency in the $\wt{O}(\epsilon^{-2})$ main-order term of Eq.\,\ref{initialbound_1} for the first case of $c_{\min} > 0$ (see Thm.\,\ref{theorem_cmin}). Furthermore, in the case of $c_{\min}=0$, the dependency on $\epsilon$ in Eq.\,\ref{initialbound_2} will be reduced from $\wt{O}(\epsilon^{-4})$ to $\wt{O}(\epsilon^{-3})$ (see Thm.\,\ref{theorem_theta}).

\section{Main Result}
\label{sec_main_result}

We first illustrate the common structure of our algorithms. As an input, we receive a desired accuracy $\epsilon \in (0,1)$, the confidence level $\delta \in (0,1]$, and the cost function $c \in [0,1]$. Since the model~$p$ is unknown, akin to~\citet{azar2013minimax} for DMDPs, we collect transition samples from the generative model and we use them to compute an estimate $\wh p$ by simply evaluating the frequency of transitions from each state-action pair $s,a$ to any state $s'$. In particular, we rely on a carefully tuned function to determine the number of transition samples that should be collected for every state-action pair. For some positive values of $X$ and $y$, we introduce the allocation function\footnote{The actual choice of $X$ and $y$ is algorithm-specific and it is illustrated later.}
\begin{align}
    \phi(X,y) := \alpha \cdot \Bigg( & \frac{X^3 \wh{\Gamma}}{y \epsilon^2} \log\left( \frac{X S A }{y \epsilon \delta} \right) + \frac{X^2 S}{y \epsilon} \log\left( \frac{X S A }{y \epsilon \delta} \right) + \frac{X^2 \wh{\Gamma}}{y^2} \log^2\left( \frac{X S A }{y \delta} \right) \Bigg)\CommaBin
\label{alloc_function}
\end{align}
where $\alpha > 0$ is a numerical constant and $\wh{\Gamma} := \max_{s,a} \norm{ \widehat{p}(\cdot\vert s,a)}_0 \leq \Gamma$ is the largest support of $\wh p$.

Let us now consider that the conditions of Lem.\,\ref{lemma_wellposedproblem} hold. A standard approach would then be to execute SSP-value iteration (\VI) on the estimated SSP-MDP $\wh M = \langle \mathcal S, \mathcal A, g, \wh p, c\rangle$ and return the corresponding optimal policy $\wh \pi$. While this approach is effective in DMDPs and finite-horizon problems, it may fail in the SSP setting. In fact, the estimated SSP-MDP may not even admit a proper optimal policy and deriving guarantees on the actual value of $\wh \pi$ (i.e., $V^{\wh \pi}$) may not be possible. As such, instead of solving the estimated SSP-MDP $\wh M$, we rather execute an extended value iteration (EVI) scheme tailored to SSP problems, which can be run efficiently as shown by \cite{tarbouriech2019no}. As detailed in App.\,\ref{app_EVI_SSP}, \EVI for SSP constructs confidence intervals for~$\wh p$ and builds a suitable SSP-MDP $\wt M = \langle \mathcal S, \mathcal A, g, \wt p, c\rangle$, where $\wt p$ belongs to the confidence intervals and is chosen so that the corresponding optimal policy $\wt \pi$ is optimistic w.r.t.\,to the optimal policy~$\pi^\star$ of $M$. More formally, let us now consider that our SSP-MDP at hand has a strictly positive cost function $c$ (which entails that the conditions of Lem.\,\ref{lemma_wellposedproblem} hold), and consider a set $\mathcal{N}$ of samples collected so far as well as a \VI precision level $\gammaVI > 0$. Then $\EVI\left(\mathcal{N}, c, \gammaVI \right)$  outputs an optimistic value vector $\wt{v}$ and an optimistic policy $\wt{\pi}$ that is greedy w.r.t.\,$\wt{v}$. Note that here (as opposed to \citealp{tarbouriech2019no}), we consider Bernstein-based concentration inequalities for \EVI, as it is done by \citet{cohen2020near} as well as in average-reward \EVI by \citet{improved_analysis_UCRL2B}. The crucial advantage of \EVI w.r.t.\,\VI run on the estimated SSP-MDP is that $\wt \pi$ is proper in $\wt M$. Indeed its value function in the optimistic model $\wt{p}$, denoted by $\wt{V}^{\wt{\pi}}$, is bounded with high-probability. This is shown by the following lemma (which stems from~\citealp[][Lem.~4 \& App.\,E]{tarbouriech2019no}), where we denote by ${V}^\star$ (resp.\,$\wt{V}^\star$) the optimal value function in the true model~$p$ (resp.\,optimistic model~$\wt p$).

\begin{lemma}\label{lem:optimism}
 For any cost function $c \geq c_{\min} > 0$, let $(\wt{v}, \wt{\pi}) = \EVI\left(\mathcal{N}, c, \gammaVI \right)$. Then with high probability, we have the component-wise inequalities $\wt{v} \leq V^{\star}$, $\wt{v} \leq \wt{V}^{\star} \leq \wt{V}^{\wt{\pi}}$, and if the \VI~precision level verifies $\gammaVI \leq \frac{c_{\min}}{2}$, then $\wt{V}^{\wt{\pi}} \leq \left(1 + \frac{2 \gammaVI}{c_{\min}}\right) \wt{v}$.
\end{lemma}
We are now ready to detail our algorithms.

\subsection{Strictly Positive Cost Function}

When $c_{\min} > 0$, we seek to achieve the standard PAC guarantees of Def.\,\ref{def_eps_opt}. The algorithm is reported in Alg.\,\ref{algo_1}. Since no prior knowledge about the optimal policy is available, the algorithm's subroutine \texttt{SEARCH} (Alg.\,\ref{algo_while_loop}) relies on a doubling scheme to guess the range of the optimal value function $B_\star$. Starting with $\Delta=1$, we use the allocation function $\phi$ to determine a sufficient number of samples to compute an $\epsilon$-optimal policy \textit{if} the range of the optimal policy was smaller than $\Delta$. We then test whether $\Delta$ is indeed a valid upper bound on the range of the optimistic value returned by \EVI and, relying on Lem.\,\ref{lem:optimism}, we stop whenever the test is successful and return $\wt \pi$. Otherwise, we double the guess $\Delta$ and reiterate. Since $\phi$ is increasing in its first argument, the total number of samples required at iteration is also increasing.

\begin{algorithm2e}[t!]%
\caption{Algorithm for $c_{\min} > 0$}
    \label{algo_1}
  \KwIn{ cost function $c$ with minimum cost $c_{\min} > 0$, accuracy $\epsilon > 0$, confidence level $\delta \in (0,1)$, allocation function $\phi(\cdot, \cdot)$.} 
  $\wt \pi := \texttt{SEARCH}(c)$. \\
  \KwOut{the policy $\wt{\pi}$.}
\end{algorithm2e}

\begin{algorithm2e}[t!]
     \caption{\texttt{SEARCH}}
      \label{algo_while_loop}
  
    \KwIn{A positive cost function $c'$.}
    Set $\iota := \min_{s,a}{c'(s,a)}$ the minimum cost of $c'$. \\
    Set $\Delta := \tfrac{1}{2}$ and \texttt{continue} = \texttt{True}, sample set $\mathcal N = \emptyset$. \\
    \While{\textup{\texttt{continue}}}{ 
  Set $\Delta \leftarrow 2 \Delta$. \\
  Add samples obtained from the generative model to $\mathcal{N}$ until $\phi(\Delta, \iota)$ samples are available at each state-action pair. \\
  Compute $(\wt{v}, \wt{\pi}) := \EVI\left(\mathcal{N}, c', \gammaVI := \frac{\iota \epsilon}{6 \Delta}\right)$ with $\mathcal{N}$ the samples collected so far. \\
  \If{$\norm{\wt{v}}_{\infty} \leq \Delta$}{
  \texttt{continue} = \texttt{False}.
  }
  }
   \KwOut{the policy $\wt{\pi}$.}
\end{algorithm2e}

\begin{theorem}\label{theorem_cmin}
For any accuracy $\epsilon\in (0,1]$, confidence $\delta\in (0,1)$, and cost function $c$ in $[c_{\min}, 1]$, with $c_{\min} > 0$, Algorithm~\ref{algo_1} (with the allocation function of Eq.\,\ref{alloc_function}) is $(\epsilon,\delta)$-optimal with a sample complexity bounded as follows
\begin{align*}
    \SampCompOne = \wt{O} \Bigg( & \frac{B_{\star}^3 \Gamma S A}{c_{\min} \epsilon^2} + \frac{B_{\star}^2 S^2 A}{c_{\min} \epsilon} + \frac{B_{\star}^2 \Gamma S A}{c_{\min}^2} \Bigg).
\end{align*}
\end{theorem}

\paragraph{Proof sketch} Throughout we assume that the standard high-probability event of satisfied concentration inequalities holds (Lem.\,\ref{lemma_high_prob}). The analysis starts by showing that enough samples per state-action pair are available to guarantee that the candidate optimistic policy $\wt \pi$ is proper not only in the optimistic model $\wt p$ but also in the true model $p$. This is proved by applying the simulation lemma for SSP of Lem.\,\ref{lemma_simulation_ssp}. Hence, the value functions in the true and optimistic models, denoted by $V^{\wt \pi}$ and $\wt V^{\wt \pi}$ respectively, are each bounded component-wise and close enough up to a multiplicative constant. Moreover, by optimism and choice of \VI accuracy $\gammaVI$, we obtain that $\wt{V}^{\wt \pi} \leq V^{\star} + O(\epsilon)$ component-wise. In addition we can prove that the doubling scheme of Alg.\,\ref{algo_1} guarantees that $\Delta \leq 2 B_{\star}$. Putting everything together implies two important properties: (i) it holds that $ \norm{V^{\wt \pi}}_{\infty} = O(B_{\star})$, and (ii) it is sufficient to bound $V^{\wt \pi} - \wt{V}^{\wt \pi}$ in order to obtain the sought-after guarantee of Def.\,\ref{def_eps_opt}. To do so, subtracting the two respective Bellman equations yields
\begin{align*}
    V^{\wt \pi}(s) - \wt{V}^{\wt \pi}(s) &= \sum_{y \in \cS} p(y \vert s, \wt \pi(s)) (V^{\wt \pi}(y) - \wt{V}^{\wt \pi}(y)) + W(s),
\end{align*}
where we introduce 
\begin{align*}
    W(s) := \sum_{y \in \cS} (p(y \vert s, \wt \pi(s)) - \wt{p}(y \vert s, \wt \pi(s))) \wt{V}^{\wt \pi}(y).
\end{align*}
Let us denote by $Q^{\wt \pi} \in \mathbb{R}^{S \times S}$ the transition matrix between the non-goal states under policy $\wt \pi$ in the true model $p$ (i.e., for any $(s,s') \in \cS, Q^{\wt \pi}(s,s') := p(s' \vert s, \wt \pi(s))$). Since $\wt{\pi}$ is proper in $p$, $Q^{\wt \pi}$ is strictly substochastic which implies that the matrix $(I-Q^{\wt \pi})$ is invertible, and therefore we have
\begin{align} \label{eq_important_sum}
    V^{\wt \pi}(s) - \wt{V}^{\wt \pi}(s) &= \left[ (I - Q^{\wt \pi})^{-1} W \right]_s = \sum_{t=0}^{+ \infty} \mathbb{E}_{\wt{\pi},p} \Big[ \mathds{1}_{s_t \neq g} W(s_t) ~ \vert s_0 = s \Big].
\end{align}
We now apply variance-aware arguments, similar to e.g., \citet{azar2013minimax, azar2017minimax, cohen2020near, improved_analysis_UCRL2B}, in order to decompose $W(s_t)$ and thus obtain
\begin{align*}
    V^{\wt \pi}(s) - \wt{V}^{\wt \pi}(s) &\leq  \text{\ding{202}} + \text{\ding{203}} + \text{\ding{204}},
\end{align*}
with
\begin{align*}
  \text{\ding{202}} &= \wt{O}\left( \sqrt{\frac{\wh{\Gamma}}{n}} \sum_{t=0}^{+ \infty} \mathbb{E}_{\wt{\pi},p} \left[ \mathds{1}_{s_t \neq g} \sqrt{\mathbb{V}(s_t)} \right] \right), \\
  \text{\ding{203}} &= \wt{O}\left( \sqrt{\frac{ \wh{\Gamma} }{n}} \sqrt{ c_{\min} \Delta} \sum_{t=0}^{+ \infty} \mathbb{P}_{\wt{\pi},p} ( s_t \neq g ) \right), \\
  \text{\ding{204}} &= \wt{O}\left( \frac{\Delta S}{n} \sum_{t=0}^{+ \infty} \mathbb{P}_{\wt{\pi},p} ( s_t \neq g )\right),
\end{align*}
where $n$ denotes the minimum number of samples collected at each state-action pair and where we define the empirical branching factor $\wh \Gamma := \max_{s,a} \norm{\wh{p}(\cdot \vert s,a)}_{0} \leq \Gamma$ as well as the following variance quantity
\begin{align*}
    \mathbb{V}(s_t) := \sum_{s' \in \cS'} p(s' \vert s_t, \wt \pi(s_t)) \left( \wt{V}^{\wt \pi}(s') - \sum_{s''\in \cS} p(s'' \vert s_t, , \wt \pi(s_t)) \wt{V}^{\wt \pi}(s'') \right)^2.
\end{align*}
The series that appears in the terms \ding{203} and \ding{204} is bounded by leveraging the exponential decay of the probability of not reaching the goal w.r.t.\,the time step $t$, i.e., 
\begin{align}\label{eq_prob_exp_decay}
    \mathbb{P}_{\wt \pi, p}\left(s_t \neq g\right) = O\left( \exp\left(-\frac{c_{\min} t}{ \norm{V^{\wt \pi}}_{\infty}}\right)\right).   
\end{align}
Furthermore, the series appearing in term \ding{202} is split by the Cauchy-Schwarz inequality and by the time step decomposition of \citet{cohen2020near} into \textit{intervals} that are carefully constructed so that the expected variance $\mathbb{V}$ accumulated over a whole interval is adequately bounded. Ultimately, there remains to invert the equation $\text{\ding{202}} + \text{\ding{203}} + \text{\ding{204}} \leq \epsilon$ w.r.t.\,$n$ in order to obtain a lower bound on $n$.

\paragraph{High-level connection to DMDP analysis} Here we point out a high-level parallel with the DMDP analysis of \citet{azar2013minimax}. Recall informally that the latter analysis handles sums of the sort $\left[ (I - \gamma P^{\wh \pi})^{-1} W \right]_s = \sum_{t=0}^{+ \infty} \gamma^t \mathbb{E}_{\wh{\pi},p} \Big[ W(s_t) ~ \vert s_0 = s \Big]$, for different yet related quantities $W$. In contrast, the SSP setting does not display a natural discount factor in the Bellman equations ($\gamma=1$ in Eq.\,\ref{bellman_eq}). Instead, the \say{shrinking} of summands in Eq.\,\ref{eq_important_sum} is captured by the indicator $\mathds{1}_{s_t \neq g}$, which is an \textit{unknown, time-dependent, state-dependent and policy-dependent} quantity. Crucially, we obtain that its \textit{expectation} displays an exponential decay similar to the $\gamma^t$-phenomenom observed in DMDPs, specifically Eq.\,\ref{eq_prob_exp_decay}. As such, we can argue that any proper policy $\pi$ in SSP displays a pseudo-discounting property with rate $\gamma_{\pi} \sim \exp{\left(-c_{\min} / \norm{V^{\pi}}_{\infty}\right)} < 1$. Despite this interesting connection with the DMDP analysis, note that one cannot simply plug in the DMDP analysis for the SSP setting considered here (recall that DMDPs are a subclass of SSP-MDPs, not the other way around). In fact, we need to consider SSP-specific analytical tools to handle the upper bounding of Eq.\,\ref{eq_important_sum} (e.g., interval decomposition) as alluded to above and detailed in App.\,\ref{investigation_minimum_value_n}.

\subsection{Any Cost Function and Restricted Optimality}
\label{subsection_cmin_0}

Whenever $c_{\min} = 0$, Alg.\,\ref{algo_1} has a possibly unbounded sample complexity. To handle this, we add a small perturbation to all the costs (denoted by $\nu$) during the computation of the optimistic policies. Note that this perturbation technique is also employed by e.g.,~\citet{bertsekas2013stochastic, tarbouriech2019no, cohen2020near}. Executing Alg.\,\ref{algo_1} with the modified cost function would directly return a policy that is $\epsilon$-optimal w.r.t.\ the optimal policy of the SSP-MDP with perturbed cost. Nonetheless, this is not a significant guarantee, since it does not say anything about the performance of the policy in the original SSP-MDP. For this reason, we rather derive $\epsilon$-optimality guarantees w.r.t.\ a set of restricted policies, as discussed in Sect.\,\ref{sec_preliminaries}. Note that Def.\,\ref{def_eps_opt.theta} involves the unknown quantities $\{D_s\}_{s \in \cS}$. In fact, in order to properly tune the cost perturbation~$\nu$ and return an $\epsilon$-optimal policy, we need to compute an upper bound $\wh D$ the SSP-diameter. This requires an additional initial phase to perform such estimation step. We explain the procedure in App.\,\ref{app_est_diam} (Alg.\,\ref{algo:subroutine}) and show that the amount of samples used in this initial phase is subsumed in the final sample complexity bound by the second phase where we compute the final candidate policy. 

The second phase is basically the same as in Alg.\,\ref{algo_1} except for the perturbation on the original cost function by $\nu$ and a slightly different precision level $\gammaVI$. The analysis simply applies Thm.\,\ref{theorem_cmin} in the \textit{perturbed model} and leverages the optimality restriction of Def.\,\ref{def_restriction} to control the bias induced by the cost perturbation (see App.\,\ref{app_proof_theta_thm}). We obtain the following sample complexity guarantees.

\begin{theorem}\label{theorem_theta}
For any accuracy $\epsilon\in (0,1]$, confidence $\delta\in (0,1)$, cost function $c$ in $[0, 1]$, and slack parameter $\theta \geq 1$, Algorithm~\ref{algo_2} (with the allocation function of Eq.\,\ref{alloc_function}) is $(\epsilon,\delta,\theta)$-optimal with a sample complexity bounded as follows
\begin{align*}
    \SampCompTwo = \wt{O} \Bigg( & \frac{\theta D B_{\star}^3 \Gamma S A }{\epsilon^3} + \frac{\theta D B_{\star} S^2 A}{\epsilon^2} + \frac{\theta^2 D^2 B_{\star}^2 \Gamma S A}{\epsilon^2} \Bigg).
\end{align*}
\end{theorem}
In Thm.\,\ref{theorem_theta} the slack parameter $\theta \geq 1$ is implicitly considered as a bounded constant which implies that $\Pi_{\theta} \subsetneq \Pi$. It is possible to link $\theta$ to the accuracy $\epsilon$, by for example instantiating $\theta = \epsilon^{-1}$. In this special case, at the cost of a worse dependency on $\epsilon$ for Thm.\,\ref{theorem_theta} (namely, in $\wt{O}(\epsilon^{-4})$), we obtain an $\epsilon$-accurate guarantee w.r.t.\,the optimal proper policy as $\epsilon$ tends to $0$, since $\lim_{\epsilon \rightarrow 0} \Pi_{\epsilon^{-1}} = \Pi$.

\begin{algorithm2e}[t!]
   \caption{Algorithm for $\theta < + \infty$}
    \label{algo_2}
  \KwIn{ slack parameter $\theta \geq 1$, accuracy $\epsilon > 0$, confidence level $\delta \in (0,1)$, cost function $c$, allocation function $\phi(\cdot, \cdot)$.} 
  First compute $\wh{D}$ an upper bound estimate of the SSP-diameter (see Alg.\,\ref{algo:subroutine} of App.\,\ref{app_est_diam}). \\
  Set cost perturbation $\nu = \frac{\epsilon}{2 \theta \wh{D}}$. \\
  $\wt{\pi} = \texttt{SEARCH}(c \lor \nu)$. \\
  \KwOut{the policy $\wt{\pi}$.}
\end{algorithm2e}

\section{Discussion} 

In this section we discuss the bounds obtained in Thm.\,\ref{theorem_cmin} and \ref{theorem_theta} and compare them with existing bounds in related settings.

First, let us consider the unit-cost case ($c_{\min}=1$). In this case, it is easy to show that the sample complexity is lower-bounded as $\Omega(S A (B_{\star})^3 / \epsilon^2)$. In fact, the inclusion DMDP $\subset$ SSP-MDP, and the mapping $1/(1-\gamma) = B_{\star}$ in the unit-cost case (see footnote\footnoteref{footnote1}), allows us to directly inherit the lower bound of \citet{azar2013minimax} in DMDPs, which scales as $\Omega(S A / (1-\gamma)^3 \epsilon^2)$. This shows that in this case, the sample complexity in Thm.\,\ref{theorem_cmin} matches the lower bound in the $\epsilon$, $A$ and $B_{\star}$ terms.

As for the dependency on the state space, our bound scales as $\Gamma S$ with $\Gamma \in [1, S+1]$ the maximal branching factor. While in many environments $\Gamma = O(1)$ as long as the dynamics are not too chaotic, $\Gamma S$ may scale with $S^2$ in the worst case. This possibly quadratic dependency in~$S$ is worse than the linear dependency for sample complexity in DMDPs with a generative model \citep{azar2013minimax}. This bound mismatch is also present between SSP-MDP and finite-horizon in the regret minimization framework, where no-regret algorithms for SSP \citep{tarbouriech2019no, cohen2020near} scale as $\wt{O}(S)$, which contrasts with the lower bound in $\sqrt{S}$ derived by \citet{cohen2020near} and with regret bounds in finite-horizon \citep[e.g.,][]{azar2017minimax} which match the $\sqrt{S}$ lower bound. How to improve the state dependency for the SSP setting remains an open question, whether it be in the regret minimization or sample complexity setting. Note that the $\Gamma S$ dependency stems from the analysis estimating the transition kernel accurately well across the state-action space. As an immediate byproduct, this implies that after its sample collection phase, each algorithm Alg.\,\ref{algo_1} or \ref{algo_2} can actually guarantee $\epsilon$-optimal planning for \textit{any} cost function in $[c_{\min}, 1]$, respectively where $c_{\min} > 0$ (Thm.\,\ref{theorem_cmin}) and where $c_{\min} = 0$ with $\theta < + \infty$ (Thm.\,\ref{theorem_theta}). 

The role of the effective horizon $H$ in finite-horizon or $1/(1-\gamma)$ in DMDPs is captured in the SSP setting by the ratio $B_{\star}/ c_{\min}$ (when $c_{\min} > 0$). Compared to the application of the simulation lemma for SSP (Lem.\,\ref{lemma_simulation_ssp}), the use of variance-aware techniques succeeds in shaving off a term $B_{\star}/ c_{\min}$ in the main order term of the sample complexity. Our analysis combines techniques on regret minimization for the SSP problem~\citep{cohen2020near} and on sample complexity of DMDPs with a generative model~\citep{azar2013minimax}. The latter work removes a factor $1 / (1-\gamma)$, with $\gamma < 1$ the discount factor. As we fleshed out in our analysis, the role of the discount factor~$\gamma$ in DMDPs is implicitly captured in SSP by the policy-dependent indicator $\mathds{1}_{\{s_t \neq g\}}$, whose expectation decays exponentially w.r.t.\,the time $t$ with rate scaling as $\exp{\left(-c_{\min} / B_{\star}\right)} < 1$. Recall that this high-level parallel between the SSP-MDP and DMDP settings is not surprising insofar as, more generally, DMDPs are a subclass of SSP-MDPs~\citep{bertsekas1995dynamic}. 

We can wonder whether the $c_{\min}$ dependency is unavoidable or not in the sample complexity result of Thm.\,\ref{theorem_cmin} for the case $c_{\min} > 0$. While we do not have a definite answer, we can investigate the question by drawing a high-level analogy with the finite-horizon setting. Taking the regret of \UCBVI \citep{azar2017minimax} in the stationary case $\sum_k (V^{\pi_k} - V^{\star}) \leq \sqrt{H^2 S A K}$, and performing a regret-to-PAC conversion, we notice that to obtain $V^{\pi} - V^{\star} \leq \epsilon$, we require $K' \approx H^2 S A / \epsilon^2$ episodes and hence $H K'$ time steps of sample complexity, since each episode accounts for $H$ interactions with the environment. We thus see that the dependency in $H^3 = H^2 H$ can be decomposed as $H^2$ capturing the range of the value function, and another $H$ capturing the length of an episode. Now by analogy, pretending such a regret-to-PAC conversion works in SSP (which we recall from Sect.\,\ref{sec_preliminaries} is not the case), we would obtain the dependency $B_{\star}^2 (B_{\star}/c_{\min})$, because the range of the optimal value function is $B_{\star}$ while the characteristic length of an optimal episode scales as $B_{\star}/c_{\min}$ in the worst case. This reasoning provides an intuitive support to our conjecture that the lower bound must contain a dependency on the characteristic length of an optimal episode which in the worst case scales as $B_{\star}/c_{\min}$. It remains an open question whether it is possible or not to construct a lower bound problem explicitly displaying the critical role of~$c_{\min}$. 
Finally, the bound of Thm.\,\ref{theorem_cmin} inherits a $\wt{O}(\epsilon^{-2})$ dependency, when $c_{\min}$ is considered as a positive constant. On the other hand, Thm.\,\ref{theorem_theta} can cope with very small (or even zero-valued) $c_{\min}$, yet the bound worsens to $\wt{O}(\epsilon^{-3})$, and the performance becomes \textit{restricted} to policies with not too large expected goal-reaching time (via the slack parameter $\theta$). This interesting behavior does not appear in the finite-horizon or discounted case (where the range of rewards has no influence on the rate in $\epsilon$), and it captures the key role of the minimum cost played in the behavior of the optimal goal-reaching policy: the more the minimum cost is allowed to be small, the longer the duration of the trajectory to reach the goal may be, thus the harder it is analysis-wise to control the trajectory variations of a policy between two models.



\newpage

\bibliography{bibliography.bib}

\appendix

\newpage{}
\onecolumn

\section{High-probability Event}

Here we characterize the high-probability event, denoted by $\mathcal{E}$. Throughout the remainder of the analysis, we will assume that $\mathcal{E}$ holds.

\begin{lemma}
\label{lemma_high_prob}
    Denote by $\mathcal{E}$ the event under which for any time step $t \geq 1$ and for any state-action pair $(s,a)$ and next state $s'$, it holds that
\begin{align}
    \abs{\widehat{p}_{t}(s' \vert s,a) - p(s' \vert s,a)} \leq 4 \sqrt{ \frac{\widehat{p}_{t}(s'\vert s,a)}{N^{+}_{t}(s,a)} \log\left(\frac{S A N^{+}_{t}(s,a)}{\delta} \right)} + \frac{28 \log\left(\frac{S A N^{+}_{t}(s,a)}{\delta} \right)}{N^{+}_{t}(s,a)},
\label{empirical_b_ineq}
\end{align}
where $N_t^{+}(s,a) := \max \{1,N_t(s,a) \} $ with $N_t$ the state-action counts accumulated up to (and including) time $t$. Then we have $\mathbb{P}(\mathcal{E}) \geq 1 - \delta$.
\end{lemma}

\begin{proof}
The confidence intervals in Eq.\,\ref{empirical_b_ineq} are constructed using the empirical Bernstein inequality, which guarantees that $\mathbb{P}(\mathcal{E}) \geq 1 - \delta$, see e.g., \cite{cohen2020near, improved_analysis_UCRL2B}.
\end{proof}

\section{Extended Value Iteration for SSP} 
\label{app_EVI_SSP}

Here we briefly recall how to perform an extended value iteration (EVI) scheme tailored to SSP, as explained by~\citet{tarbouriech2019no}. Note that we leverage a Bernstein-based construction of confidence intervals, as done by e.g.,~\citet{improved_analysis_UCRL2B, cohen2020near}.

Formally, we consider as input any SSP-MDP instance $M$ (cf.\,Def.\,\ref{def:ssp.mdp.app}) with cost function $c$ lower bounded by $c_{\min} > 0$, a set $\mathcal{N}$ of samples collected so far (with corresponding state-action counters denoted by $N$) and a \VI precision level $\gammaVI > 0$. We now detail what the scheme $\EVI\left(\mathcal{N}, c, \gammaVI \right)$ does and outputs.

First it computes a set of plausible SSP-MDPs defined as $$\mathcal{M} := \{ \langle \mathcal{S}, \mathcal{A}, g, \widetilde{p}, c \rangle ~\vert ~ \widetilde{p}(g|g,a) =~1,~ \widetilde{p}(s'|s,a) \in \mathcal{B}(s,a,s'),~ \sum_{s' \in \cS'} \wt{p}(s' \vert s,a) = 1\},$$ where for any $(s,a) \in \cS \times \cA$, $\mathcal{B}(s,a,s')$ is a high-probability confidence set on the dynamics of the true SSP-MDP $M$. Specifically, we define the compact sets $\mathcal{B}(s,a,s') := [ \wh{p}(s' \vert s,a) - \beta(s,a,s'), \wh{p}(s' \vert s,a) + \beta(s,a,s')] \cap [0, 1]$, where $\beta(s,a,s')$ denotes the right hand side of Eq.\,\ref{empirical_b_ineq}. From Lem.\,\ref{lemma_high_prob} the choice of $\beta(s,a,s')$ guarantees that $M \in \mathcal{M}$ with probability at least $1-\delta$. 

Once $\mathcal{M}$ has been computed, the scheme applies extended value iteration (\EVI) to compute a policy with lowest optimistic value. Formally, it defines the extended optimal Bellman operator $\widetilde{\mathcal{L}}$ such that for any vector $\wt{v} \in \mathbb{R}^{S}$ and non-goal state $s \in \cS$,
\begin{align*}
    \wt{\mathcal{L}}\wt{v}(s) := \min_{a \in \cA} \Big\{ c(s,a) + \min_{\widetilde{p}\in \mathcal{B}(s,a)} \sum_{s' \in  \cS} \wt{p}(s' \vert s,a) \wt{v}(s') \Big\}.
\end{align*}
We consider an initial vector $\wt{v}_0 := 0$ and set iteratively $\wt{v}_{i+1} := \wt{\mathcal{L}} \wt{v}_{i}$. For the predefined \VI precision $\gammaVI > 0$, the stopping condition is reached for the first iteration $j$ such that $\norm{\wt{v}_{j+1} - \wt{v}_{j} }_{\infty} \leq \gammaVI$. The policy $\wt{\pi}$ is then selected to be the optimistic greedy policy w.r.t.\,the vector $\wt{v}_j$. While $\wt{v}_j$ is not the value function of $\wt{\pi}$ in the optimistic model $\wt{p}$, which we denote by $\wt{V}^{\wt{\pi}}$, both quantities can be related according to Lem.\,\ref{lem:optimism}.

\section{Useful Results}
\label{app_useful_results}

\begin{proof}\textbf{of Lem.\,\ref{lemma_simulation_ssp}} ~ Although we consider here a slightly different $\mathcal{P}_{\eta}^{(p)}$ set, the result is almost identical to \citet[][App.\,C]{tarbouriech2020improved}. For completeness, we report the full derivation here. The analysis follows the proof of \citet[][Lem.\,B.4]{cohen2020near} whose result can be seen as a special case of Lem.\,\ref{lemma_simulation_ssp}. First, let us assume that $\pi$ is proper in the model $p'$. This implies that its value function, denoted by $V'$, is bounded component-wise. Moreover, for any state $s \in \cS$, the Bellman equation holds as follows
\begin{align}
    V'(s) &= c(s, \pi(s)) + \sum_{y \in \cS} p'(y \vert s, \pi(s)) V'(y) \nonumber \\
    &= c(s, \pi(s)) + \sum_{y \in \cS} p(y \vert s, \pi(s)) V'(y) + \sum_{y \in \cS} \left( p'(y \vert s, \pi(s)) - p(y \vert s, \pi(s)) \right) V'(y).
\label{eq_sim_lemma}
\end{align}
By successively using H\"older's inequality and that $p' \in \mathcal{P}_{\eta}^{(p)}$ and $c(s, \pi(s)) \geq c_{\min}$, we get
\begin{align*}
    V'(s) \geq c(s, \pi(s)) - \eta \norm{V'}_{\infty} + p(\cdot \vert s, \pi(s)) ^\top V' \geq c(s, \pi(s)) \Big(1 - \frac{\eta \norm{V'}_{\infty}}{c_{\min}}\Big) + p(\cdot \vert s, \pi(s)) ^\top V'.
\end{align*}
Let us now introduce the vector $V'' := \left(1 - \frac{\eta \norm{V'}_{\infty}}{c_{\min}}\right)^{-1} V'$. Then for all $s \in \cS$,
\begin{align*}
    V''(s) \geq c(s,\pi(s)) + p(\cdot \vert s,\pi(s)) ^\top V''.
\end{align*}
Hence, from Lem.\,\ref{lemma_technical_lemma_ssp}, $\pi$ is proper in $p$ (i.e., $V < + \infty$), and we have
\begin{align}
    V \leq V'' \leq \left(1 + 2 \frac{\eta \norm{V'}_{\infty}}{c_{\min}} \right) V',
    \label{eq_sim_lemma_1}
\end{align}
where the last inequality stems from condition \eqref{eq_key_sim_lemma_ssp} and the fact that $\frac{1}{1-x} \leq 1+2x$ holds for any $0 \leq x \leq \frac{1}{2}$. Conversely, analyzing Eq.\,\ref{eq_sim_lemma} from the other side, we get
\begin{align*}
    V'(s) \leq c(s, \pi(s)) \left( 1 + \frac{\eta \norm{V'}_{\infty}}{c_{\min}} \right) + p(\cdot \vert s, \pi(s)) ^\top V'.
\end{align*}
Let us now introduce the vector $V'' := \left(1 + \frac{\eta \norm{V'}_{\infty}}{c_{\min}}\right)^{-1} V'$. Then
\begin{align*}
    V''(s) \leq c(s,\pi(s)) + p(\cdot \vert s,\pi(s)) ^\top V''.
\end{align*}
We then obtain in the same vein as Lem.\,\ref{lemma_technical_lemma_ssp} (by leveraging the monotonicity of the Bellman operator $\mathcal{L}^{\pi}U(s) := c(s,\pi(s)) + p(\cdot \vert s,\pi(s)) ^\top U$) that $V'' \leq V$, and therefore
\begin{align}
    V' \leq \left(1 + \frac{\eta \norm{V'}_{\infty}}{c_{\min}} \right) V.
    \label{eq_sim_lemma_2}
\end{align}
Combining Eq.\,\ref{eq_sim_lemma_1} and \ref{eq_sim_lemma_2} yields component-wise
\begin{align*}
    \norm{V - V'}_{\infty} \leq 2 \frac{\eta \norm{V'}_{\infty}}{c_{\min}} \norm{V'}_{\infty} + \frac{\eta \norm{V'}_{\infty}}{c_{\min}} \norm{V}_{\infty} \leq 7 \frac{\eta \norm{V'}_{\infty}^2}{c_{\min}},
\end{align*}
where the last inequality stems from plugging condition \eqref{eq_key_sim_lemma_ssp} into Eq.\,\ref{eq_sim_lemma_1}.

Note that here $p$ and $p'$ play symmetric roles; we can perform the same reasoning in the case where $\pi$ is proper in the model $p$ and it would yield an equivalent result by switching the dependencies on $V$ and $V'$.
\end{proof}
\begin{lemma}[\citealp{bertsekas1991analysis}, Lem.\,1]
    Consider an SSP instance under the conditions of Lem.\,\ref{lemma_wellposedproblem}. Let $\pi$ be any policy, then
    \begin{itemize}
        \item If there exists a vector $U: \cS \rightarrow \mathbb{R}$ such that $U(s) \geq c(s, \pi(s)) + \sum_{s' \in \cS} p(s' \vert s, \pi(s)) U(s')$ for all $s \in \cS$, then $\pi$ is proper, and $V^{\pi}$ the value function of $\pi$ is upper bounded by $U$ component-wise, i.e., $V^{\pi}(s ) \leq U(s)$ for all $s \in \cS$.
        \item If $\pi$ is proper, then its value function $V^{\pi}$ is the unique solution to the Bellman equations $V^{\pi}(s ) = c(s, \pi(s)) + \sum_{s' \in \cS} p(s' \vert s, \pi(s)) V^{\pi}(s' )$ for all $s \in \cS$.
    \end{itemize}
\label{lemma_technical_lemma_ssp}
\end{lemma}
We now state a useful result which showcases the exponential decay of the goal-reaching probability of a proper policy with component-wise bounded value function.

\begin{lemma}[\citealp{cohen2020near}, Lem.\,B.5]\label{lemma_quantify_goal_reaching_probability}
Let $\pi$ be a proper policy such that for some $d > 0$, $V^{\pi}(s) \leq d$ for every non-goal state $s$. Then the probability that the cumulative cost of $\pi$ to reach the goal state from any state $s$ is more than $m$, is at most $2 e^{-m/(4 d)}$ for all $m \geq 0$. Note that a cost of at most $m$ implies that the number of steps is at most $m / c_{\min}$.
\end{lemma}
We finally spell out an important property stemming from optimism.

\begin{lemma}\label{useful_lemma_bound_vtilde_delta}
    Under the event $\mathcal{E}$, we have $\wt{V} \leq V^{\star} + \frac{\epsilon}{3}$ component-wise.
\end{lemma}

\begin{proof}
Denote by $\wt{v}$ the \VI vector output by the computation of the candidate policy~$\wt{\pi}$ via \EVI. From Lem.\,\ref{lem:optimism} and by the choice of the \VI precision $\gammaVI := \frac{\epsilon c_{\min}}{6 \Delta}$, we have component-wise that $\wt{V} \leq \left( 1 + \frac{2\gammaVI}{c_{\min}}\right) \wt{v} \leq V^{\star} + \frac{\epsilon}{3 \Delta} \wt{v} \leq V^{\star} + \frac{\epsilon}{3}$ since $\wt{v} \leq \Delta$ by construction of Alg.\,\ref{algo_while_loop}.
\end{proof}


\subsection{Procedure to Estimate an Upper Bound of the SSP-Diamater}
\label{app_est_diam}

\begin{lemma}[\textsc{$D$-subroutine}]\label{lemma_delta_subroutine}
    With probability at least $1-\delta$, the \textsc{$D$-subroutine} (Alg.\,\ref{algo:subroutine}):
    \begin{itemize}[leftmargin=.2in,topsep=-4pt,itemsep=0pt,partopsep=0pt, parsep=0pt]
\item has a sample complexity bounded by $\wt{O}\left( D^2 \Gamma S A / \epsilon^2 + D S^2 A / \epsilon \right)$,
\item requires at most $\log_2\left(D(1+\epsilon)\right) + 1$ inner iterations,
\item outputs a quantity $\wh{D}$ that verifies $D \leq \wh{D} \leq \left( 1 + 2 \epsilon (1+\epsilon) \right) (1+\epsilon) D$.
\end{itemize}
\end{lemma}
\begin{algorithm2e}[t!]
   \caption{\textsc{$D$-Subroutine}}
    \label{algo:subroutine}
  \KwIn{ accuracy $\epsilon > 0$, confidence level $\delta \in (0,1)$.} 
  Set $W := \tfrac{1}{2}$ and $\norm{\wt{v}}_{\infty} := 1$. \\
  \While{$\norm{\wt{v}}_{\infty} > W$}{
  Set $W \leftarrow 2 W$.\\
  Set the accuracy $\eta := \frac{\epsilon}{W}$.\\
  Collect additional samples until $\wh p \in \mathcal{P}_{\eta / 2}$ with confidence level $\delta$ (we verify this using the Bernstein upper bound of Eq.\,\ref{empirical_b_ineq}) \\
  Compute $(\wt{v},\_) := \EVI\left(\mathcal{N}, c=1, \gammaVI := \frac{ \epsilon}{2}\right)$.
  }
  \KwOut{the optimistic quantity $\wh D := \left( 1 + 2 \eta \norm{\wt{v}}_{\infty} \right) \norm{\wt{v}}_{\infty}$.}
\end{algorithm2e}
We now delve into the analysis of the \textsc{$\wh D$-subroutine}. Throughout the remainder of the proof, we will assume that the event $\mathcal{E}$ holds. We now give a useful statement stemming from optimism.

\begin{lemma}\label{lemma_useful}
At any stage of the \textsc{$\wh D$-subroutine}, denote by $\wt{v}$ the vector computed using \EVI for SSP (App.\,\ref{app_EVI_SSP}). Then under the event $\mathcal{E}$, we have component-wise (i.e., starting from any non-goal state) that $\wt{v} \leq \min_{\pi} V^{\pi}_p \leq D$.
\end{lemma}
\begin{proof}
    The first inequality stems from Lem.\,\ref{lem:optimism} while the second inequality uses the definition of the SSP-diameter $D$ and the fact that the considered costs are equal to 1.
\end{proof}
We now prove Lem.\,\ref{lemma_delta_subroutine}. Denote by $i$ the iteration index of the subroutine (starting at $i=1$), so that $W_i = 2^i$. Introduce $j := \min \{i \geq 1: \norm{\wt{v}_i}_{\infty} \leq W_i \}$. By choice of each optimistic model, we have $\norm{\wt{v}_i}_{\infty} \leq D$ at any iteration $i \geq 1$ from Lem.\,\ref{lemma_useful}. Since $(W_i)_{i \geq 1}$ is a strictly increasing sequence, the subroutine is bound to end in a finite number of iterations (i.e., $j < + \infty$), and given that $W_{j-1} \leq \norm{\wt{v}_{j-1}}_{\infty} \leq D$, we get $j \leq \log_2\left(D\right) + 1$. Moreover, we have $\norm{\wt{v}_{j}}_{\infty} \leq W_{j}$ and $\eta_j = \frac{\epsilon}{W_j}$, which implies that $\eta_j \leq \frac{\epsilon}{\norm{\wt{v}_{j}}_{\infty}}$. Moreover, combining $W_{j-1} \leq D$ and $W_{j-1} = \frac{W_j}{2} = \frac{\epsilon}{2 \eta_j}$ yields that $\frac{\epsilon}{2D} \leq \eta_j$. The Bernstein upper bound of Eq.\,\ref{empirical_b_ineq} entails that the total sample complexity is bounded by $\wt{O}\left( D^2 \Gamma S A / \epsilon^2 + D S^2 A / \epsilon \right)$. Now, denote by $\wt{v}$ the optimistic matrix output by the \textsc{$\wh D$-subroutine}. Let us consider $s_1 \in \arg\max_{s} \min_{\pi} \mathbb{E}\left[\tau_{\pi}(s) \right]$. Denote by $\wt{\pi}$ the greedy policy w.r.t.\,the vector $\wt{v}$ in the optimistic model. Then we have
\begin{align*}
    D = \min_{\pi} \mathbb{E}\left[\tau_{\pi}(s_1) \right] \leq \mathbb{E}\left[\tau_{\wt{\pi}}(s_1) \right] &\myineeqa \left( 1 + 2 \eta \norm{\mathbb{E}\left[\wt{\tau}_{\wt{\pi}}\right]}_{\infty} \right) \mathbb{E}\left[\wt{\tau}_{\wt{\pi}}(s_1 ) \right] \\
    &\myineeqb \left( 1 + 2 \eta (1+\epsilon) \norm{\wt{v}}_{\infty} \right) (1+\epsilon) \wt{v}(s_1 ) \\
    &\leq \left( 1 + 2 \eta (1+\epsilon) \norm{\wt{v}}_{\infty} \right) (1+\epsilon) \norm{\wt{v}}_{\infty} := \wh D \\
    &\myineeqc \left( 1 + 2 \eta (1+\epsilon) \norm{\wt{v}}_{\infty} \right) (1+\epsilon) D \\
    &\myineeqd \left( 1 + 2 \epsilon (1+\epsilon) \right) (1+\epsilon) D,
\end{align*}
where (a) corresponds to the simulation lemma for SSP (Lem.\,\ref{lemma_simulation_ssp}), (b) comes from the value iteration precision $\gammaVI := \frac{\epsilon}{2}$ which implies that $\mathbb{E}\left[\wt{\tau}_{\wt{\pi}}\right] \leq (1+2\gammaVI)\wt{v} \leq (1+\epsilon)\wt{v}$ component-wise according to Lem.\,\ref{lem:optimism}, (c) is implied by Lem.\,\ref{lemma_useful}, and finally (d) uses that $\eta \norm{\wt v}_{\infty} \leq \epsilon$ as proved above.


\newpage

\section{Proof of Thm.\,\ref{theorem_cmin}}
\label{investigation_minimum_value_n}

Here we provide the proof of Thm.\,\ref{theorem_cmin}. Denoting by the subscript $i$ the quantities considered at any iteration~$i$ of Alg.\,\ref{algo_while_loop}, recall that the algorithm terminates at the first iteration $i$ such that \mbox{$\norm{\wt{v}_i}_{\infty} \leq \Delta_i$}. This implies that at the previous iteration $i-1$, $\Delta_{i-1} < \norm{\tilde{v}_{i-1}}_{\infty}$. Also, $\Delta_{i} = 2 \Delta_{i-1}$ and $\norm{\tilde{v}_{i-1}}_{\infty} \leq B_{\star}$ from optimism. Combining everything gives $\Delta_{i} \leq 2B_{\star}$. Therefore when Alg.\,\ref{algo_2} terminates it is aware of a quantity $\Delta := \Delta_i$ such that $\norm{\wt{v}}_{\infty} \leq \Delta \leq 2 B_{\star}$.

We denote by $n$ the minimum number of samples collected at each state-action pair. We denote by $\wt{\pi}$ the candidate policy output by Alg.\,\ref{algo_1}. Let us denote by $V$ and $\wt{V}$ the value functions of policy $\wt{\pi}$ in the true model $p$ and the optimistic model $\wt{p}$, respectively (note that we may have $V = + \infty$ for some components if $\wt{\pi}$ is not proper in $p$).

Note that $p := p(\cdot \vert \cdot, \wt{\pi}(\cdot))$, $\wh{p}:= \wh{p}(\cdot \vert \cdot, \wt{\pi}(\cdot))$ and $\wt{p}:= \wt{p}(\cdot \vert \cdot, \wt{\pi}(\cdot))$ can be seen as matrices. Our analysis draws inspiration from variance-aware techniques, see e.g., \cite{azar2013minimax, azar2017minimax, improved_analysis_UCRL2B, cohen2020near}. We will make multiple use of the Cauchy-Schwartz inequality, for which we will use the symbol $\ineqcs$. We assume throughout that the event $\mathcal{E}$ holds. Finally, we introduce the (unknown) quantity $\Gamma := \max_{s,a} \norm{p(\cdot \vert s,a)}_{0}$, and its empirical counterpart $\wh \Gamma := \max_{s,a} \norm{\wh{p}(\cdot \vert s,a)}_{0}$ (note that we always have $\wh \Gamma \leq \Gamma$).

We first require to have $\wt{p} \in \mathcal{P}_{\eta}^{(p)}$ with accuracy $\eta = \frac{c_{\min}}{6 \Delta}$. To do so, we use the triangle inequality to write $\abs{ \wt{p} - p} \leq \abs{ \wt{p} - \wh{p}} + \abs{ \wh{p} - p}$. The second term is bounded by the empirical Bernstein inequality (Eq.\,\ref{empirical_b_ineq}), and the first term is bounded the same way by construction of \EVI. Hence, by inverting Eq.\,\ref{empirical_b_ineq} to extract $n$ and after some algebraic manipulations (i.e., by applying the technical lemma of \citealp[][Lem.\,8]{kazerouni2017conservative}), is it sufficient to require
\begin{empheq}[box=\fbox]{align}
  n = \Omega \left( \frac{\Delta^2 \wh{\Gamma}}{c_{\min}^2} \log^2\left( \frac{\Delta S A}{\delta c_{\min}}\right) + \frac{\Delta}{c_{\min}} \log\left( \frac{\Delta S A}{\delta c_{\min}}\right) \right).
\tag{$\alpha$}
\label{n_1}
\end{empheq}
The simulation lemma (Lem.\,\ref{lemma_simulation_ssp}) then ensures that $\wt{\pi}$ is proper in $p$, and moreover that its value function verifies $V \leq 2 \Delta$ component-wise by virtue of Lem.\,\ref{useful_lemma_bound_vtilde_delta}. Since $\wt{\pi}$ is proper in both $p$ and $\wt{p}$, the associated Bellman equations hold, thus entailing the following for any non-goal state $s$
\begin{align*}
    V(s) - \wt{V}(s) &= \sum_{y \in \cS} p(y \vert s) V(y) - \sum_{y \in \cS} \wt{p}(y \vert s) \wt{V}(y) \\
    &= \sum_{y \in \cS} p(y \vert s) (V(y) - \wt{V}(y)) + \sum_{y \in \cS} (p(y \vert s) - \wt{p}(y \vert s)) \wt{V}(y).
\end{align*}
Let us define
\begin{align*}
    W(s) := \sum_{y \in \cS} (p(y \vert s) - \wt{p}(y \vert s)) \wt{V}(y).
\end{align*}
Note that $W(g) = 0$. Denote by $Q \in \mathbb{R}^{S \times S}$ the transition matrix restricted between the non-goal states of policy $\wt \pi$ in the true model $p$, i.e., for any $(s,s') \in \cS^2, Q(s,s') := p(s' \vert s, \wt \pi(s))$. Since $\wt{\pi}$ is proper in $p$, the matrix $Q$ is strictly substochastic which implies that the matrix $(I-Q)$ is invertible, and therefore we have
\begin{align}
    V(s) - \wt{V}(s) &= \left[ (I - Q)^{-1} W \right]_s \nonumber \\
    &= \sum_{t=0}^{+ \infty} \mathbb{E}_{\wt{\pi},p} \Big[ \mathds{1}_{s_t \neq g} W(s_t) \quad \vert s_0 = s \Big] \nonumber.
\end{align}
First, let us consider that $V(s) \leq \wt{V}(s)$. Then from Lem.\,\ref{useful_lemma_bound_vtilde_delta} we immediately have that $V(s) \leq V^{\star}(s) + \frac{\epsilon}{3}$. From now on, we thus consider that $V(s) \geq \wt{V}(s)$. Hence we have
\begin{align}
    V(s) - \wt{V}(s) &\leq \sum_{t=0}^{+ \infty} \mathbb{E}_{\wt{\pi},p} \Big[ \mathds{1}_{s_t \neq g} \abs{W(s_t)} \quad \vert s_0 = s \Big]. \label{important_series}
\end{align}
From now on, for notational simplicity, we will omit the (implicit) dependency $s_0 = s$ for the expectations. We bound each term $\abs{W(s_t)}$. Given that $\wt{V}(g) = 0$ and both $p(\cdot \vert s)$ and $\wt{p}(\cdot \vert s)$ are probability distributions over $\mathcal{S}'$, the \say{shifting} trick (also performed by e.g., \citealp{improved_analysis_UCRL2B,jin2019adversarial,cohen2020near}) yields
\begin{align*}
    W(s_t) &=  \sum_{y \in \cS'} (p(y \vert s_t) - \wt{p}(y \vert s_t)) \left(\wt{V}(y) - \sum_{z \in \cS} p(z \vert s_t) \wt{V}(z)\right).
\end{align*}
In addition the empirical Bernstein inequality entails that there exist two absolute positive constants $c_1$ and $c_2$ such that $\abs{p(s' \vert s_t) - \wt{p}(s' \vert s_t )} \leq c_1 \sqrt{\frac{\wh{p}(s' \vert s_t)\log(S'A \delta^{-1} n)}{n}} + c_2 \frac{\log(S'A \delta^{-1} n)}{n}$ (see e.g., \citealp[][Thm.\,10]{improved_analysis_UCRL2B}). Recall that $S' = S+1$ amounts to the total number of states (i.e., the $S$ non-goal states plus the goal state $g$). Setting $Z(s_t) := \sum_{z \in \cS} p(z \vert s_t) \wt{V}(z)$, we have
\begin{align}
    \abs{W(s_t)} &\leq \sum_{s' \in \cS'}  \abs{\wt{V}(s') - Z(s_t)} \cdot \abs{p(s' \vert s_t) - \wt{p}(s' \vert s_t )} \nonumber \\
    &\leq c_1 \sum_{s' \in \cS'} \sqrt{ \frac{\wh{p}(s' \vert s_t) \left( \abs{\wt{V}(s') - Z(s_t)} \right)^2 \log(S'A \delta^{-1} n)}{n}} + 2c_2 \sum_{s' \in \cS'} \frac{\Delta \log(S'A \delta^{-1} n)}{n} \nonumber \\
    &\ineqcs c_1 \sqrt{\frac{\wh{\Gamma} \log(S'A \delta^{-1} n)}{n}} \sqrt{ \sum_{s' \in \cS'}  \wh{p}(s' \vert s_t) \left( \abs{\wt{V}(s') - Z(s_t)} \right)^2 } + 2c_2 \sum_{s' \in \cS'} \frac{\Delta\log(S'A \delta^{-1} n)}{n} \nonumber \\
    &\leq  c_1 \sqrt{\frac{  \log(S'A \delta^{-1} n)}{n}} \sqrt{  \left\vert \sum_{s' \in \cS'} \left( \wh{p}(s' \vert s_t) - p(s' \vert s_t) \right) 4 \Delta^2       \right\vert      }                     \nonumber \\
    &\quad + c_1 \sqrt{\frac{ \wh{\Gamma} \log(S'A \delta^{-1} n) \mathbb{V}(s_t) }{n}} + 2c_2 \frac{\Delta S' \log(S'A \delta^{-1} n)}{n}, \label{bound_Y}
\end{align}
where we use the subadditivity of the square root and define the following variance
\begin{align*}
    \mathbb{V}(s_t) := \sum_{s' \in \cS'} p(s' \vert s_t) \left( \wt{V}(s') - \sum_{s''\in \cS} p(s'' \vert s_t) \wt{V}(s'') \right)^2.
\end{align*}
Leveraging \eqref{n_1} which guarantees that $\wh{p} \in \mathcal{P}_{\eta}^{(p)}$ with accuracy $\eta = \frac{c_{\min}}{6 \Delta}$, the first term in Eq.\,\ref{bound_Y} can be bounded as $c_1 \sqrt{\frac{ \wh{\Gamma} \log(S'A \delta^{-1} n)}{n}} \Delta \sqrt{\frac{c_{\min}}{6 \Delta}}$. Consequently, plugging the bound of Eq.\,\ref{bound_Y} into Eq.\,\ref{important_series} yields
\begin{align*}
    V(s) - \wt{V}(s) &\leq  \text{\ding{202}} + \text{\ding{203}} + \text{\ding{204}},
\end{align*}
where
\begin{align*}
  \text{\ding{202}} &:= c_1 \sqrt{\frac{\wh{\Gamma} \log(S'A \delta^{-1} n)}{n}} \sum_{t=0}^{+ \infty} \mathbb{E}_{\wt{\pi},p} \left[ \mathds{1}_{s_t \neq g} \sqrt{\mathbb{V}(s_t)} \right], \\
  \text{\ding{203}} &:= c_1 \sqrt{\frac{ \wh{\Gamma} \log(S'A \delta^{-1} n)}{n}} \Delta \sqrt{\frac{c_{\min}}{6 \Delta}} \sum_{t=0}^{+ \infty} \mathbb{P}_{\wt{\pi},p} ( s_t \neq g ), \\
  \text{\ding{204}} &:= c_2 \frac{\Delta S' \log(S'A \delta^{-1} n)}{n} \sum_{t=0}^{+ \infty} \mathbb{P}_{\wt{\pi},p} ( s_t \neq g ).
\end{align*}
Leveraging that $V \leq 2 \Delta$ component-wise, we obtain that $\mathbb{P}_{\wt{\pi},p} ( s_t \neq g ) \leq 2 \exp\left( - \frac{c_{\min} t}{8 \Delta} \right)$ by applying Lem.\,\ref{lemma_quantify_goal_reaching_probability} with $m = c_{\min} t$. To make an analogy to the infinite-horizon discounted setting studied by \citet{azar2013minimax}, we can observe that we have $\mathbb{P}_{\wt{\pi},p} ( s_t \neq g ) \sim \gamma^t$ where $\gamma \sim \exp{\left(-\frac{c_{\min}}{\Delta}\right)} < 1$.
\begin{align*}
    \sum_{t=0}^{+ \infty} \mathbb{P}_{\wt{\pi},p} ( s_t \neq g ) \leq \frac{2}{1 - \exp\left( -\frac{c_{\min}}{8 \Delta}\right)} = \frac{2 \exp\left( \frac{c_{\min}}{8 \Delta}\right)}{\exp\left( \frac{c_{\min}}{8 \Delta}\right) - 1} \leq \frac{19 \Delta}{c_{\min}},
\end{align*}
where the last inequality uses that $e^x \geq 1 + x$ holds for any real $x$. Consequently, we get
\begin{align*}
    \text{\ding{204}} \leq \frac{19 c_2 \Delta^2 S' \log(S'A \delta^{-1} n)}{c_{\min} n}.
\end{align*}
We seek to ensure that $\text{\ding{204}} \leq \frac{2\epsilon}{9}$. There simply remains to invert the inequality above to extract $n$ and do some algebraic manipulations (see e.g., \citealp[][Lem.\,9]{kazerouni2017conservative}). We thus require that:
\begin{empheq}[box=\fbox]{align}
  n = \Omega \left(  \frac{\Delta^2 S}{c_{\min} \epsilon} \log\left( \frac{\Delta S A}{c_{\min} \epsilon \delta}\right) \right)
\tag{$\beta$}
\label{n_2}
\end{empheq}
Furthermore, we have
\begin{align*}
    \text{\ding{203}} \leq 19 c_1 \frac{\Delta}{c_{\min}} \sqrt{\frac{ \wh{\Gamma} \log(S'A \delta^{-1} n)}{n}} \Delta \sqrt{\frac{c_{\min}}{6 \Delta}}.
\end{align*}
We seek to ensure that $\text{\ding{203}} \leq \frac{2\epsilon}{9}$. There simply remains to invert the inequality above to extract $n$ and do some algebraic manipulations (see e.g., \citealp[][Lem.\,9]{kazerouni2017conservative}). We thus require that:
\begin{empheq}[box=\fbox]{align}
  n = \Omega \left(  \frac{\Delta^3 \wh{\Gamma}}{c_{\min} \epsilon^2} \log\left( \frac{\Delta S A}{c_{\min} \epsilon \delta}\right) \right).
  \tag{$\gamma$}
  \label{n_3}
\end{empheq}
We now proceed in bounding \ding{202}. To do so, we split the time into \textit{intervals}, similar to \citet{cohen2020near}. The first interval begins at the first time step, and each interval ends when its total cost accumulates to at least $\Delta$ (or when the goal state $g$ is reached). Denote by $t_m$ the time step at the beginning of the $m$-th interval, and by $H_m$ the length of the $m$-th interval. An important property is that $H_m \leq 2 \Delta / c_{\min}$. Denote by $\mathbb{I}_m$ the boolean equal to 1 if the goal $g$ is not reached by the end of the $m$-th interval, and denote by $s_{(m)}$ the state at the end of the $m$-th interval. Note that~$\mathbb{I}_m = 1 \iff s_{(m)} \neq g$, implying that $\mathbb{E}_{\wt{\pi},p}\left[\mathbb{I}_m \right] = \mathbb{P}_{\wt{\pi},p}(s_{(m)} \neq g)$. We introduce a change of variable in the sums, from the time index $t$ to the interval index $m$. Formally, for any time index~$t$, there exists an interval $m+1$ (during which it occurs) and an integer $h \in [H_{m+1}]$ such that $t = \sum_{i=0}^{m} H_i + h$. The change of variable yields the following
\begin{align}
    \sum_{t=0}^{+ \infty} \mathbb{E}_{\wt{\pi},p} \left[ \mathds{1}_{s_t \neq g} \sqrt{\mathbb{V}(s_t)} \right] &\leq \sum_{m=0}^{+ \infty} \mathbb{E}_{\wt{\pi},p} \left[ \mathbb{I}_m \sum_{h = t_{m+1}}^{t_{m+1} + H_{m+1}} \sqrt{\mathbb{V}(s_h)} \right] \nonumber \\
    &\ineqcs \sum_{m=0}^{+ \infty} \mathbb{P}_{\wt{\pi},p}(s_{(m)} \neq g) \sqrt{ \mathbb{E}_{\wt{\pi},p} \left[ \left( \sum_{h = t_{m+1}}^{t_{m+1} + H_{m+1}} \sqrt{\mathbb{V}(s_h)} \right)^2 \right] } \nonumber \\
    &\ineqcs \sum_{m=0}^{+ \infty} \mathbb{P}_{\wt{\pi},p}(s_{(m)} \neq g) \sqrt{ \mathbb{E}_{\wt{\pi},p} \left[ H_{m+1} \sum_{h = t_{m+1}}^{t_{m+1} + H_{m+1}} \mathbb{V}(s_h) \right] } \nonumber \\
    &\leq \sqrt{\frac{2\Delta}{c_{\min}}} \sum_{m=0}^{+ \infty} \mathbb{P}_{\wt{\pi},p}(s_{(m)} \neq g) \sqrt{ \mathbb{E}_{\wt{\pi},p} \left[ \sum_{h = t_{m+1}}^{t_{m+1} + H_{m+1}} \mathbb{V}(s_h) \right] }. \label{eq0}
\end{align}
To bound the expression above, we first use the property shown by \citet[][Lem.\,4.7]{cohen2020near} that whenever every state-action pair has been sampled sufficiently many times (specifically, at least $\alpha B_{\star} S c_{\min}^{-1} \log( B_{\star} S A c_{\min}^{-1} \delta^{-1})$ times for some constant $\alpha > 0$, which is the case here), then the expected variance $\mathbb{V}$ accumulated over a whole interval $m$ can be bounded as $O(B_{\star}^2)$, i.e., there exists an absolute constant $c_3 > 0$ such that
\begin{align}\label{eq1}
    \mathbb{E}_{\wt{\pi},p} \left[ \sum_{h = t_{m+1}}^{t_{m+1} + H_{m+1}} \mathbb{V}(s_h) \right] \leq c_3 \Delta^2.
\end{align}
Second, we bound the series of the probabilities. The construction of the intervals entails that if the $m$-th interval does not end in the goal state, then the cumulative cost to reach the goal state is more than $\Delta m$. Furthermore, the probability of the latter event can be bounded by Lem.\,\ref{lemma_quantify_goal_reaching_probability} leveraging the component-wise inequality $V \leq 2\Delta$. As a result, we get
\begin{align*}
    \mathbb{P}_{\wt{\pi},p}(s_{(m)} \neq g) \leq  2 \exp\left( - \frac{\Delta m}{8 \Delta}\right) = 2 \exp\left(-\frac{1}{8}\right)^m,
\end{align*}
which implies that
\begin{align}\label{eq2}
    \sum_{m=0}^{+ \infty} \mathbb{P}_{\wt{\pi},p}(s_{(m)} \neq g) \leq \frac{2}{1 - \exp(-\frac{1}{8})}.
\end{align}
Plugging Eq.\,\ref{eq1} and \ref{eq2} into Eq.\,\ref{eq0} gives
\begin{align*}
    \text{\ding{202}} \leq 25 c_1 \sqrt{c_3} \sqrt{\frac{\Delta^{3/2} \wh{\Gamma} \log(S'A \delta^{-1} n)}{c_{\min} n}}.
\end{align*}
We seek to ensure that $\text{\ding{202}} \leq \frac{2\epsilon}{9}$. There simply remains to invert the inequality above to extract $n$ and do some algebraic manipulations (see e.g., \citealp[][Lem.\,9]{kazerouni2017conservative}). We thus require (once again) that:
\begin{empheq}[box=\fbox]{align}
  n = \Omega \left( \frac{\Delta^3 \wh{\Gamma}}{c_{\min} \epsilon^2} \log\left( \frac{\Delta S A }{c_{\min} \epsilon \delta} \right) \right).
\tag{$\gamma$}
\label{n_3}
\end{empheq}

\vspace{0.05in}%
\noindent Overall, combining the requirements of Eq.\,\eqref{n_1}, \eqref{n_2} and \eqref{n_3} means that we get the component-wise guarantee that $V \leq \wt{V} + \frac{2\epsilon}{3}$, and therefore from Lem.\,\ref{useful_lemma_bound_vtilde_delta} that $V \leq V^{\star} + \epsilon$, as soon as:
\begin{empheq}[box=\fbox]{align*}
    n = \Omega \left(  \frac{\Delta^3 \wh{\Gamma}}{c_{\min} \epsilon^2} \log\left( \frac{\Delta S A }{c_{\min} \epsilon \delta} \right) + \frac{\Delta^2 S}{c_{\min} \epsilon} \log\left( \frac{\Delta S A }{c_{\min} \epsilon \delta} \right) + \frac{\Delta^2 \wh{\Gamma}}{c_{\min}^2} \log^2\left( \frac{\Delta S A }{c_{\min} \delta} \right) \right).
\end{empheq}

\vspace{0.1in}

\section{Proof of Thm.\,\ref{theorem_theta}}\label{app_proof_theta_thm}

Here we provide the proof of Thm.\,\ref{theorem_theta} by establishing that the output policy $\wt{\pi}$ of Alg.\,\ref{algo_2} is $\epsilon$-optimal w.r.t.\,the restricted set $\Pi_{\theta}$. We assume that the event $\mathcal{E}$ holds. Here we offset all the costs with the additive perturbation $\nu = \frac{\epsilon}{2\theta \wh D}$. We use the subscript $\nu$ to denote quantities considered in the \textit{perturbed model}. In the perturbed model, the costs are set to $c'_{\nu}(s,a) := \max\{c(s,a), \nu\}$, which in particular implies that the minimum cost verifies $\min_{s,a} c'_{\nu}(s,a) \geq \nu$. 

The application of Thm.\,\ref{theorem_cmin} in the perturbed model immediately yields the component-wise inequality $V_{\nu} \leq V_{\nu}^{\star} + \frac{\epsilon}{2}$. Moreover, let $\pi^{\S} \in \min_{\pi \in \Pi_{\theta}} V^{\pi}$, $V^{\S} := V^{\pi^{\S}}$ and $T^{\S} := \mathbb{E}\left[ \tau_{\pi^{\S}} \right]$. In particular, we have $V^{\star}_{\nu} \leq V^{\S}_{\nu}$ and $T^{\S}(s) \leq \theta D_{s} \leq \theta \wh D$. Furthermore, given the choice of $\nu$ and the fact that $c'_{\nu}(s,a) \leq c(s,a) + \nu$, we have $V^{\S}_\nu \leq V^{\S} + \nu T^{\S} \leq V^{\S} + \frac{\epsilon}{2}$. Lastly, we have $V \leq V_{\nu}$. Putting everything together yields the sought-after inequality $V \leq V^{\S} + \epsilon$.

\end{document}